\documentclass[11pt,twocolumn]{IEEEtran}
\usepackage{setspace}
\usepackage{amsmath,amssymb,mathrsfs}    
\usepackage{graphicx}   
\usepackage{verbatim}   
\usepackage{color}      
\usepackage{epstopdf}
\usepackage{float}
\usepackage{cite}
\usepackage{amssymb,amsthm}
\usepackage{graphicx}
\usepackage[font=footnotesize]{subfig}

\newtheorem{prob-statement}{Problem}

\newtheorem{lemma}{Lemma}

\DeclareMathOperator*{\argmin}{arg\,min}
\DeclareMathOperator*{\minimize}{minimize}
\DeclareMathOperator*{\st}{subject \ to}

\begin{document}

\title{On Estimating Multi-Attribute Choice Preferences using Private Signals and Matrix Factorization}

\author{Venkata Sriram Siddhardh Nadendla and Cedric Langbort}

\author{
\IEEEauthorblockN{{\large Venkata Sriram Siddhardh Nadendla and Cedric Langbort}}
\\[1ex]
\IEEEauthorblockA{Coordinated Science Laboratory, \\ 
University of Illinois at Urbana-Champaign. \\
Email: \{nadendla, langbort\}@illinois.edu}
}

\maketitle

\begin{abstract}
Revealed preference theory studies the possibility of modeling an agent's revealed preferences and the construction of a consistent utility function. However, modeling agent's choices over preference orderings is not always practical and demands strong assumptions on human rationality and data-acquisition abilities. Therefore, we propose a simple generative choice model where agents are assumed to generate the choice probabilities based on latent factor matrices that capture their choice evaluation across multiple attributes. Since the multi-attribute evaluation is typically hidden within the agent's psyche, we consider a signaling mechanism where agents are provided with choice information through private signals, so that the agent's choices provide more insight about his/her latent evaluation across multiple attributes. We estimate the choice model via a novel multi-stage matrix factorization algorithm that minimizes the average deviation of the factor estimates from choice data. Simulation results are presented to validate the estimation performance of our proposed algorithm.
\end{abstract}
%

\section{Motivation}

In this paper, we consider the problem of modeling and estimating choice preferences across multiple attributes in a non-parametric manner via active data acquisition, where private signals are used to learn hidden preferences across multiple attributes. Such designs will significantly reduce search friction in current recommendation systems which are designed to model the average behavior of a population of agents. Furthermore, such choice models can be used in modern cyber-social-physical systems (e.g. smart transportation systems), where persuasive signals can be designed to steer agents' decisions towards a social optimum. 

One of the first attempts to model human decisions was made by Samuelson in his seminal work in \cite{Samuelson1938} in 1938 using \emph{the theory of revealed preferences}. This model is based on the idea that human agents choose decisions to maximize their utility functions. Later, several models have been proposed with utility-maximization as a cornerstone philosophy to characterize human decisions. Some well-accepted models in the literature are Von Neumann and Morgenstern's expected utility theory (EUT) \cite{Book-VonNeumann}, Herbert Simon's \emph{bounded rationality} models \cite{Book-Rubinstein1998} and random utility based discrete choice models \cite{Book-Train}. However, the problem of finding utility functions that are consistent with the revealed choices remained open until Afriat proposed a constructive algorithm in \cite{Afriat1967} to compute utility functions from finite choice revelations. 

Several parametric choice models have been proposed to capture a person's choice evaluation based on multiple attributes in the past literature. The characterization of multi-attribute choice models can be broadly classified into two types, depending on how evaluations across multiple attributes are combined together by the agent before picking a choice. In the first type of models, utility functions across multiple attributes are linearly combined together using appropriate weights, which are estimated from data. Such models have been extensively studied in the context of parametric models, with one of the most accepted models being the conditional multinomial logit (MNL) model proposed by McFadden in \cite{McFadden1973}. Alternatively, in the second type of models, conditional preferences across multiple attributes are combined in the realm of choice probabilities. Some examples include generalized extreme-value (GEV) models \cite{Book-Train} and the mixture-MNL (MMNL) model \cite{Rusmevichientong2014,Kallus2016}. 

Although utility maximization has provided a tractable framework to design systems for human agents, several experiments have been documented in the psychology literature where human agents demonstrate deviating behaviors from the framework of utility maximization. For example, principles of \emph{transitivity} \cite{Tversky1969} and \emph{substitutability} \cite{Tversky1969b} have been violated by human agents under various choice settings. Therefore, we avoid the utilitarian regime and other structural assumptions considered traditionally to model human decisions, and propose a non-parametric choice probability model where the agent is assumed to evaluate choices based on controller's signals which provide information about choices in terms of multiple attributes. Such a choice model facilitates the design of signaling mechanisms that strategically influence agents' decisions in a desired manner. 

The design of signaling mechanisms have a wide range of applications. For example, consider the example of smart transportation system (Alice), where a commuter (Bob) chooses a route from the set of all routes that connect start and destination nodes. On the other hand, the transportation agencies (e.g. GPS industry) are interested in controlling network congestion via providing traffic signals. While public signals are constructed to serve a population of commuters, private signals are designed to leverage an individual's cognitive preferences across the set of all routes. Since private signals outperform public signals in most practical cases \cite{Hellwig2002}, there is a greater potential to influence commuters and control traffic congestion via private signaling. 


\section{Literature Survey \& Contributions \label{sec: Contributions}}


Non-parametric choice models have been proposed by several researchers in the past. The most well-known one is the rank-based choice model where the agent is assumed to pick a choice with the highest rank. The agent's decisions are modeled using a demand function which is characterized by the probability mass function over the set of all possible orders of choice probabilities. This model was first proposed by Block and Marschak \cite{Block1960} in 1960, and was later studied and analyzed in various applications. The caveat in this model is that it is practically impossible to observe the entire choice ordering in any decision-theoretic context. A possible remedy for this issue was provided by Jagabathula and Shah in \cite{Jagabathula2011}, where they have considered the problem of modeling choice orders in the presence of constrained data. Later, Farias \emph{et al.} have used this choice model to solve a revenue prediction problem in \cite{Farias2012,Farias2013}. However, it is not always practical to obtain choice rankings from human decision makers. For example, such queries are not relevant to the decision framework and can unnecessarily irk human agents. 

Therefore, we consider a choice model where we are concerned with revealed choices rather than preference orderings. This is primarily because preference orders cannot be practically observed in every decision-theoretic framework, and is also not needed in system which only rely on choice probabilities. Furthermore, people evaluate choices across multiple attributes using different utility functions which need not be compatible with each other. In an attempt to address all these issues, we assume that the choice preferences (conditioned on the choice information signaled strategically to the agent) are characterized by a matrix factorization model based on stochastic matrices, where the matrix factors capture the agent's choice evaluation across multiple attributes. In order to find the optimal matrix factors in our choice model, we also propose a novel multi-stage estimation procedure which relies on choice revelations that are collected after the agent receives the strategic signal regarding choice information.


Matrix factorization has been a very active topic for several decades, and has emerged as a powerful tool to analyze clusters and other features in datasets. In contrast to traditional approaches such as singular value decomposition (SVD) and LU decomposition, the framework of non-negative matrix factorization (NMF) 
has become a versatile tool for dimensionality reduction and has been used in various contexts and applications. For more details, the readers may refer to a comprehensive review on NMF in \cite{Wang2013b}. Stochastic matrix factorization (SMF) is a constrained NMF, where a stochastic matrix is approximated as a product of two stochastic matrices. Although several generative probabilistic models (e.g. Latent Dirichlet Allocation models \cite{Blei2003}) have been studied using various inference methods, the application of SMF for estimating these generative models was first proposed by Arora \emph{et al.} in the context of topic modeling in \cite{Arora2012} and have studied uniqueness in SMF in the presence of separability conditions. More recently, Adams has studied the SMF problem in \cite{Adams2016}, and investigated necessary and sufficient conditions on the observed data for a unique factorization. In addition, Adams has also derived bounds on the parameters for any observed data and presented a consistent least squares estimator. For details about the various SMF algorithms proposed in literature, the reader may refer to \cite{Luo2017} and references within.

The main drawback of SMF is that the factors are not necessarily unique in general. Although the authors in \cite{Arora2012} and \cite{Adams2016} present conditions for the existence of a unique solution, these conditions need not necessarily hold true in general, in the context of choice modeling. Therefore, uniqueness of our matrix factorization approach is not guaranteed, and can affect the estimation performance. This analysis is beyond the scope of this paper, and will be considered in our future work. Instead, we focus our attention on our novel estimation procedure that finds stochastic factors in our choice model with the aid of information signaling in a greedy fashion. 

\section{Choice Model \& Problem Setup \label{sec: Model}}
Consider two interacting agents, Alice and Bob, where Bob picks choices from a choice set $\mathbb{C} = \{ 1, \cdots, K \}$ independently across time, and Alice's goal is to evaluate Bob's choice model based on conditional preferences across a given set of attributes $\mathbb{A} = \{1, \cdots, L\}$. 
We assume that Alice constructs private signals regarding the choice set across all possible subsets of the attribute set $\mathbb{A}$. These signals can be represented as a message matrix $\mathbb{M}$ as shown in Figure \ref{Fig: message matrix}, where each column corresponds to choice information based on a specific attribute subset $A \subseteq \mathbb{A}$. Based on this assumption, we propose an active data acquisition procedure where Alice constructs a message $m(A)$ for Bob based on a chosen attribute subset $A \subseteq \mathbb{A}$, by choosing the corresponding column in the matrix $\mathbb{M}$. For the sake of clarity, we denote the message set as $\mathbb{M} = \{m_1, \cdots, m_M \}$, although each message in $\mathbb{M}$ can be uniquely mapped to a subset of $\mathbb{A}$. This notation is employed to differentiate our message labels from the labels of attribute subsets, even though $M = 2^L$. Furthermore, we also assume that Alice sends a null message to Bob (denoted as $m_1$) when $A_1 = \phi$. 

\begin{figure}[!t]
	\centering
    \includegraphics[width=0.33\textwidth]{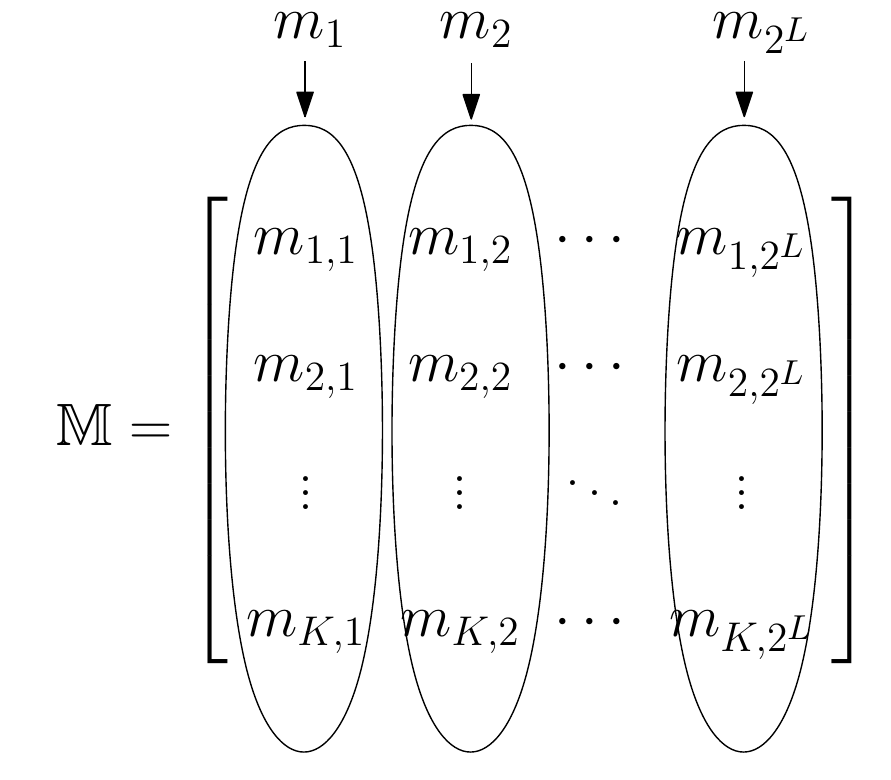}
    \caption{Signaling scheme employed by Alice}
    \label{Fig: message matrix}
\end{figure}

Having received a message signal $m_t = m$, Bob picks a choice $c_t = c$ with probability $p_{m,c} = \mathbb{P}(c_t = c | m_t = m)$, where $p_{m,c}$ is computed as
\begin{equation}
	p_{m,c} = \displaystyle \sum_{A \in 2^{\mathbb{A}}} w_{m,A} \cdot q_{A,c}
	\label{Eqn: choice-model}
\end{equation}
for all $c \in \mathbb{C}$, where $w_{m,A} = \mathbb{P}(A | m_t = m)$ is the weight that Bob assigns to the attribute set $A \subseteq \mathbb{A}$ based on the received message $m$, and $q_{A,c} = \mathbb{P}(c_t = c | A)$ is the probability that Bob would have preferred the choice $c$ based on the attribute set $A \subseteq \mathbb{A}$. The choice model in Equation \eqref{Eqn: choice-model} can also be presented equivalently in matrix form as 
\begin{equation}
	P = W Q,
	\label{Eqn: choice-model-matrix-form}
\end{equation}
where $P$ is a $M \times K$ stochastic matrix, $W$ is a $M \times 2^L$ stochastic matrix and $Q$ is a $2^L \times K$ stochastic matrix. In other words, if $u_k$ is a vector of $k$ ones for any positive integer $k$, we have
\begin{equation}
	P u_K = u_{M}, \ W u_{2^L} = u_{M}, \mbox{ and } Q u_K = u_{2^L}.
\end{equation}

We assume that the data acquisition experiment is designed in a sequence of stages, where Alice picks an attribute set $A_i$ in the $i^{th}$ stage based on an increasing order of the size of attribute subsets in $\mathbb{A}$, until she exhausts all the subsets of $\mathbb{A}$. Although it does not matter which one is chosen between two subsets of same size, we assume 
\begin{equation}
\begin{array}{lcl}
	A_1 = \phi
	\\[1ex]
	A_2 = \{1\}, & \cdots, & A_{L+1} = \{L\}
	\\[1ex]
	A_{L+2} = \{1,2\}, & \cdots, & A_{L^2+1} = \{L-1, L\}
	\\
	\vdots
	\\
	A_{2^L} = \mathbb{A}.
\end{array}
\label{Eqn: Attribute-Subset-Ordering}
\end{equation}


Generally, Bob evaluates the choices over all possible attribute sets in $2^{\mathbb{A}}$, since Alice does not reveal the attribute set $A$ based on which recommendations are constructed. However, in this paper, we assume that Alice reveals the attribute set to Bob at every stage. Since Bob has the autonomy to either completely/partially accept Alice's recommendation, or even ignore it altogether, our choice model in Equation \eqref{Eqn: choice-model} reduces to
\begin{equation}
	p_{m,c} = \displaystyle \sum_{B \subseteq A_i} w_{m,B} \cdot q_{B,c},
	\label{Eqn: choice-model-2}
\end{equation}
where Alice constructs its message $m$ based on the subset $A_i$ in the $i^{th}$ stage, and $B$ is a subset of $A_i$ over which Bob may evaluate its choice $c$. Furthermore, in the matrix representation of the choice model as given in Equation \eqref{Eqn: choice-model-matrix-form}, the weight matrix in Equation \eqref{Eqn: choice-model-matrix-form} takes the following lower-triangular form at the end of the data-acquisition process:
\begin{equation}
	W = \left[
	\begin{array}{c c c c}
		1 & 0 & \cdots & 0
		\\
		w_{2,1} & 1 - w_{2,1} & \cdots & 0
		\\
		\vdots & \vdots & \ddots & \vdots
		\\
		w_{M,1} & w_{M,2} & \cdots & 1 - \displaystyle \sum_{i = 1}^{2^L} w_{M,i}
	\end{array}
	\right],
	\label{Eqn: W structure}
\end{equation} 
with $w_{i,j} = 0$ if $A_j \nsubseteq A_i$.

Alice's goal is to find the factor-estimates $\hat{W}$ and $\hat{Q}$ for given data-set $\mathbb{P} = \{P(1), \cdots, P(N)\}$ of $N$ stochastic matrices, that best fits to the choice model in Equation \eqref{Eqn: choice-model-matrix-form}. In practice, Alice can obtain the data-set $\mathbb{P}$ by partitioning the revealed preference dataset into $N$ partitions, and estimating $P(n)$ using the data from the $n^{th}$ partition for all $n = 1, \cdots, N$. Note that the signaling mechanism and the procedure of computing $\mathbb{P}$ are abstracted out from the rest of this paper, in order to focus on the estimation of $Q$ and $W$.

Let us define the \emph{average deviation} from the model parameters $Q$ and $W$ to the data set $\mathbb{P}$ as
\begin{equation}
	f(W,Q) = \displaystyle \frac{1}{N} \sum_{n = 1}^N || WQ - P(n)||_F^2,
	\label{Eqn: deviation}
\end{equation}
where $||\Pi||_F$ is the Frobenius norm of the matrix $\Pi$. Then, the problem can be formally stated as follows.
\begin{equation}
\begin{array}{ll}
	\displaystyle \minimize_{W,Q} & f(W, Q)
	\\[2ex]
	\st & \text{1. } W u_{2^L} = u_M
	\\[2ex]
	& \text{2. } Q u_K = u_{2^L}
	\\[2ex]
	& \text{3. } \boldsymbol{0}_{M \times 2^L} \leq W \leq \boldsymbol{1}_{M \times {2^L}}
	\\[2ex]
	& \text{4. } \boldsymbol{0}_{2^L \times K} \leq Q \leq \boldsymbol{1}_{2^L \times K}.
	\\[2ex]
	& \text{5. } W \text{ satisfies Equation \eqref{Eqn: W structure}.}
\end{array}
\tag{P}
\label{Problem: General}
\end{equation}

Let $\hat{Q}$ and $\hat{W}$ denote the solution to Problem \eqref{Problem: General}. Note that the above optimization problem is nonconvex since the objective function has a bilinear structure. 
We decompose Problem \eqref{Problem: General} into multiple stages and present an approximate learning procedure that is inspired from the sequential signaling policy employed by Alice. 


\section{Sequential Learning Algorithm \label{sec: Algorithm}}
For the sake of convenience, we employ the following notation in the rest of the paper. We represent the matrices $P$, $Q$ and $W$ as 
\begin{equation}
\begin{array}{lcr}
	P = \left[
	\begin{array}{c}
		\boldsymbol{p}_1^T
		\\
		\vdots
		\\
		\boldsymbol{p}_M^T
	\end{array}
	\right], &
	Q = \left[
	\begin{array}{c}
		\boldsymbol{q}_1^T
		\\
		\vdots
		\\
		\boldsymbol{q}_{2^L}^T
	\end{array}
	\right] \mbox{ and} &
	W = \left[
	\begin{array}{c}
		\boldsymbol{w}_1^T
		\\
		\vdots
		\\
		\boldsymbol{w}_M^T
	\end{array}
	\right],
\end{array}
\end{equation}
where $\boldsymbol{p}_m^T$,, $\boldsymbol{q}_m^T$ and $\boldsymbol{w}_m^T$ are the $m^{th}$ rows in $P$, $Q$ and $W$ respectively. 

Given a message $\boldsymbol{m} \in \mathbb{M}$ based on some attribute subset $A \subseteq \mathbb{A}$, $P(\boldsymbol{\sigma} | \boldsymbol{m})$ depends on the terms $P(\boldsymbol{\sigma} | B)$ and $w(B | \boldsymbol{m})$, for all $B \subseteq A$. This motivates Alice to adopt the ordering of messages, which is identical to the ordering of alphabet subsets in Equation \eqref{Eqn: Attribute-Subset-Ordering}, so that the weight estimates from the previous stages can be used to reduce Problem \eqref{Problem: General} into a ordered sequence of $2^L$ convex programs, where each convex program (stage) corresponds to each message in the set of all possible messages at Alice. In the $i^{th}$ stage, Alice collects revealed preference data from Bob corresponding to the $i^{th}$ message $m_i$ in $\mathbb{M}$. 

\paragraph*{\textbf{Stage-1}} In the first stage of the algorithm, Alice's message is based out of a null subset of attributes. This is equivalent to the case where the agent does not provide any information regarding the attributes. In this special case, since $\boldsymbol{w}_1^T = [1 \ 0 \ \cdots \ 0]$, the first rows of both $P$ and $Q$ shall be identical. This can be illustrated in the following sub-problem:

\begin{equation}
\begin{array}{ll}
	\displaystyle \minimize_{ \boldsymbol{q}_1} & \displaystyle \frac{1}{N} \sum_{n = 1}^N  || \boldsymbol{q}_1 - \boldsymbol{p}_1(n) ||^2
	\\[4ex]
	\st & \text{1. } \displaystyle \boldsymbol{q}_1^T u_K = 1
	\\[2ex]
	& \text{2. } \boldsymbol{0} \leq \boldsymbol{q}_1 \leq \boldsymbol{1}.
\end{array}
\tag{P1}
\label{Problem: Stage-1}
\end{equation}
Problem \ref{Problem: Stage-1} is a least-squares problem, and therefore, has a closed form solution which is given by
\begin{equation}
	\hat{\boldsymbol{q}}_1 = \displaystyle \frac{1}{N} \sum_{n = 1}^N \boldsymbol{p}_1(n).
	\label{Eqn: q-1-estimate}
\end{equation}


\paragraph*{\textbf{Stage-i}} Now, let us consider the $i^{th}$ stage where Alice collects Bob's revealed preferences conditioned on the message $m_i$, for any $i = 2, \cdots, 2^L$. In this stage, the choice model that is of interest is given by
\begin{equation}
	\boldsymbol{p}_i^T = \displaystyle \boldsymbol{w}_i^T Q
\end{equation}
for all $i = 2, \cdots, 2^L$. Due to the lower-triangular structure of the weight matrix $W$, $\boldsymbol{p}_i$ is reduced to
\begin{equation}
	\boldsymbol{p}_i = \displaystyle \sum_{j = 1}^{i-1} w_{i,j} \boldsymbol{q}_j + \left( 1 - \sum_{j = 1}^{i-1} w_{i,j} \right) \boldsymbol{q}_i.
\end{equation}
Note that, in the above model, the vectors $\boldsymbol{q}_1, \cdots, \boldsymbol{q}_{i-1}$ have been estimated in the previous stages. Therefore, we equivalently represent the above model as
\begin{equation}
	\boldsymbol{p}_i = \displaystyle \Phi_i \boldsymbol{x}_i,
	\label{Eqn: Linearized Choice Model - Intermediate Stage}
\end{equation}
where 
\begin{equation}
	\Phi_i = \displaystyle \left[
	\begin{array}{c c c c}
		\boldsymbol{q}_1 & \cdots & \boldsymbol{q}_{i-1} & I_K
	\end{array}
	\right],
\end{equation}
denotes the coefficient matrix (known at the current stage since we have estimates $\boldsymbol{\hat{q}}_1, \cdots, \boldsymbol{\hat{q}}_{i-1}$ from the previous stages) in the $i^{th}$ stage, and,
\begin{equation}
	\boldsymbol{x}_i = \displaystyle \left[ 
	\begin{array}{c}
		w_{i,1}
		\\[1ex]
		\vdots
		\\[1ex]
		w_{i,i-1}
		\\[1ex]
		\displaystyle \left( 1 - \sum_{j = 1}^{i-1} w_{i,j} \right) \boldsymbol{q}_i
	\end{array}
	\right]
\end{equation}
denotes the vector to be estimated in this stage.

While the above representation in Equation \eqref{Eqn: Linearized Choice Model - Intermediate Stage} seems linear, the nonlinearity is hidden in the bilinear term $\displaystyle \left( 1 - \sum_{j = 1}^{i-1} w_{i,j} \right) \boldsymbol{q}_i$ in $\boldsymbol{x}_i$. In our proposed approximate algorithm, we ignore this bilinear nature of $\boldsymbol{x}_i$ and evaluate the solution to the following linear program:
\begin{equation}
\begin{array}{ll}
	\displaystyle \minimize_{\boldsymbol{x}_i} & \displaystyle \frac{1}{N} \sum_{n = 1}^N || \hat{\Phi}_i \boldsymbol{x}_i - \boldsymbol{p}_i(n) ||^2
	\\[4ex]
	\st & \text{1. } \displaystyle \boldsymbol{u}_{K+i-1}^T  \boldsymbol{x}_i = 1
	\\[2ex]
	& \text{2. } \boldsymbol{0} \leq \boldsymbol{x}_i \leq \boldsymbol{1},
\end{array}
\tag{P2}
\label{Problem: Stage-i}
\end{equation}
where
$\hat{\Phi}_i = \displaystyle \left[
\begin{array}{c c c c}
	\hat{\boldsymbol{q}}_1 & \cdots & \hat{\boldsymbol{q}}_{i-1} & I_K
\end{array}
\right]$ is the coefficient matrix whose entries are estimated in the previous stages.

Let $\hat{\boldsymbol{x}}_i$ denote the solution to Problem \ref{Problem: Stage-i}. If $\hat{\Phi}_i^{-1}$ is the Moore–Penrose pseudo-inverse of the matrix $\hat{\Phi}_i$, $\hat{\boldsymbol{x}}_i$ is given by
\begin{equation}
	\hat{\boldsymbol{x}}_i = \displaystyle \frac{1}{N} \sum_{n = 1}^N \hat{\Phi}_i^{-1} \boldsymbol{p}_i(n). 
\end{equation}
If $\hat{\Phi}_i^{-1}$ cannot be computed, the problem can also be solved using efficient convex optimization algorithms. Given the solution vector $\hat{\boldsymbol{x}}_i$, we find the estimates of $\{w_{i,1}, \cdots, w_{i,i-1}\}$ and $\boldsymbol{q}_i$, as shown below.
\begin{eqnarray}
	\hat{w}_{i,j} & = & \hat{\boldsymbol{x}}_i(j), \mbox{ for all } j = 1, \cdots, i-1,
	\label{Eqn: w-i,j-estimate}
	\\[2ex]
	\hat{\boldsymbol{q}}_i & = & \displaystyle \frac{1}{1 - \displaystyle \sum_{j = 1}^{i-1} \hat{\boldsymbol{x}}_i(j)} \ \boldsymbol{x}_i(i:K+i-1)
	\label{Eqn: q-i-estimate}
\end{eqnarray}


Next, we analyze the existence of a unique solution to Problem \eqref{Problem: General}, and the optimality of our proposed algorithm.

\begin{lemma}
Problem \eqref{Problem: General} does not have a unique optimal solution. Furthermore, if $\mathcal{S}$ denotes the set of all possible optimal solutions to Problem \eqref{Problem: General}, then the solution $(\hat{W},\hat{Q})$ delivered by our proposed algorithm lies in $\mathcal{S}$.
\end{lemma}
\begin{proof}
Consider the average deviation $f(W,Q)$, defined in Equation \eqref{Eqn: deviation}. Due to the lower-triangular structure of $W$ matrix, we find that $||WQ - P(n)||_F^2$ can be expanded as 
\begin{equation}
	\begin{array}{l}
		||WQ - P(n)||_F^2 
		\\[1ex]
		\quad = \displaystyle \sum_{i = 1}^M \left|\left|\displaystyle \sum_{j = 1}^{i-1} w_{i,j} \boldsymbol{q}_j + \left( 1 - \sum_{j = 1}^{i-1} w_{i,j} \right) \boldsymbol{q}_i - \boldsymbol{p}_i(n)\right|\right|^2.
	\end{array}
\end{equation}

Rearranging the order of summation in $f(W,Q)$, we have
\begin{equation}
	f(W,Q) = \displaystyle \sum_{i = 1}^M \tilde{f}_i(\boldsymbol{w}_i, \boldsymbol{q}_1, \cdots, \boldsymbol{q}_i),
\end{equation}
where
\begin{equation}
	\begin{array}{l}
		\tilde{f}_i(\boldsymbol{w}_i, \boldsymbol{q}_1, \cdots, \boldsymbol{q}_i)
		\\[1ex]
		= \displaystyle \frac{1}{N} \sum_{n = 1}^N \left|\left|\displaystyle \sum_{j = 1}^{i-1} w_{i,j} \boldsymbol{q}_j + \left( 1 - \sum_{j = 1}^{i-1} w_{i,j} \right) \boldsymbol{q}_i - \boldsymbol{p}_i(n)\right|\right|^2.
	\end{array}
\end{equation}


Given that each term in the summation is non-negative in the above representation, we have
\begin{equation}
\begin{array}{l}
	\displaystyle \min_{W, Q} \ \sum_{i = 1}^M \tilde{f}_i(\boldsymbol{w}_i, \boldsymbol{q}_1, \cdots, \boldsymbol{q}_i)
	\\[2ex]
	\qquad \qquad = \ \displaystyle \sum_{i = 1}^M \min_{W, Q} \tilde{f}_i(\boldsymbol{w}_i, \boldsymbol{q}_1, \cdots, \boldsymbol{q}_i),
\end{array}
\label{Eqn: Simultaneous minimization}
\end{equation}
if all the problems within the summation on the right hand side of Equation \eqref{Eqn: Simultaneous minimization} are carried out  simultaneously. This is equivalent to the case where we solve the following system of equations simultaneously, where each equation corresponds to a minimized term in the summation: 
\begin{equation}
	\boldsymbol{q}_1^* = \displaystyle \frac{1}{N} \sum_{n = 1}^N p_1(n),
\end{equation}
\begin{equation}
	\sum_{j = 1}^{i-1} w_{i,j}^* \boldsymbol{q}_j^* + \left( 1 - \sum_{j = 1}^{i-1} w_{i,j}^* \right) \boldsymbol{q}_i^* = \displaystyle \frac{1}{N} \sum_{n = 1}^N p_i(n),
\end{equation}
for all $i = 2, \cdots, M$. 

Since there are more variables than equations in the above system of simultaneous equations, we do not have a unique solution to this problem. Note that our proposed algorithm presents one of the solution candidates to the above system of simultaneous equations, since it presents
$$(\hat{\boldsymbol{w}}_i, \hat{\boldsymbol{q}}_i) = \displaystyle \argmin_{\boldsymbol{w}_i, \boldsymbol{q}_i} \ \tilde{f}_i(\boldsymbol{w}_i, \hat{\boldsymbol{q}}_1, \cdots, \hat{\boldsymbol{q}}_{i-1}, \boldsymbol{q}_i)$$
in the $i^{th}$ stage for all $i = 1, \cdots, M$.
\end{proof}



Given that the size of the message set $\mathbb{M}$ increases exponentially with the number of attributes, the algorithm has exponential complexity in terms of the size of the attribute set. However, we have closed-form expressions to both $\hat{Q}$ and $\hat{W}$, as given in Equations \eqref{Eqn: q-1-estimate}, \eqref{Eqn: w-i,j-estimate} and \eqref{Eqn: q-i-estimate}. Furthermore, the hierarchical structure in the power set of $\mathbb{A}$ provides us with multiple non-interfering stages which reduces to $L$ (or, equivalently $\log_2 M$) effective stages with the aid of parallel processing. 


\section{Results and Discussion \label{sec: Results}}

Consider an example setting where there are $L = 2$ attributes, and therefore, a total of $M = 4$ possible message signals. We choose the stochastic matrices $Q$ and $W$ at random so that they satisfy the structural constraints. Having chosen the matrices $Q$ and $W$, we compute the $P$ matrix using our choice model in Equation \eqref{Eqn: choice-model-matrix-form}. We sample $C$ choices from $P$ matrix, which are distributed across $N$ bins to compute the input data $P(1), \cdots, P(N)$ to our algorithm. 

\begin{figure}[!t]
	\centering
    \includegraphics[trim={7mm 0 7mm 0}, clip, width=0.5\textwidth]{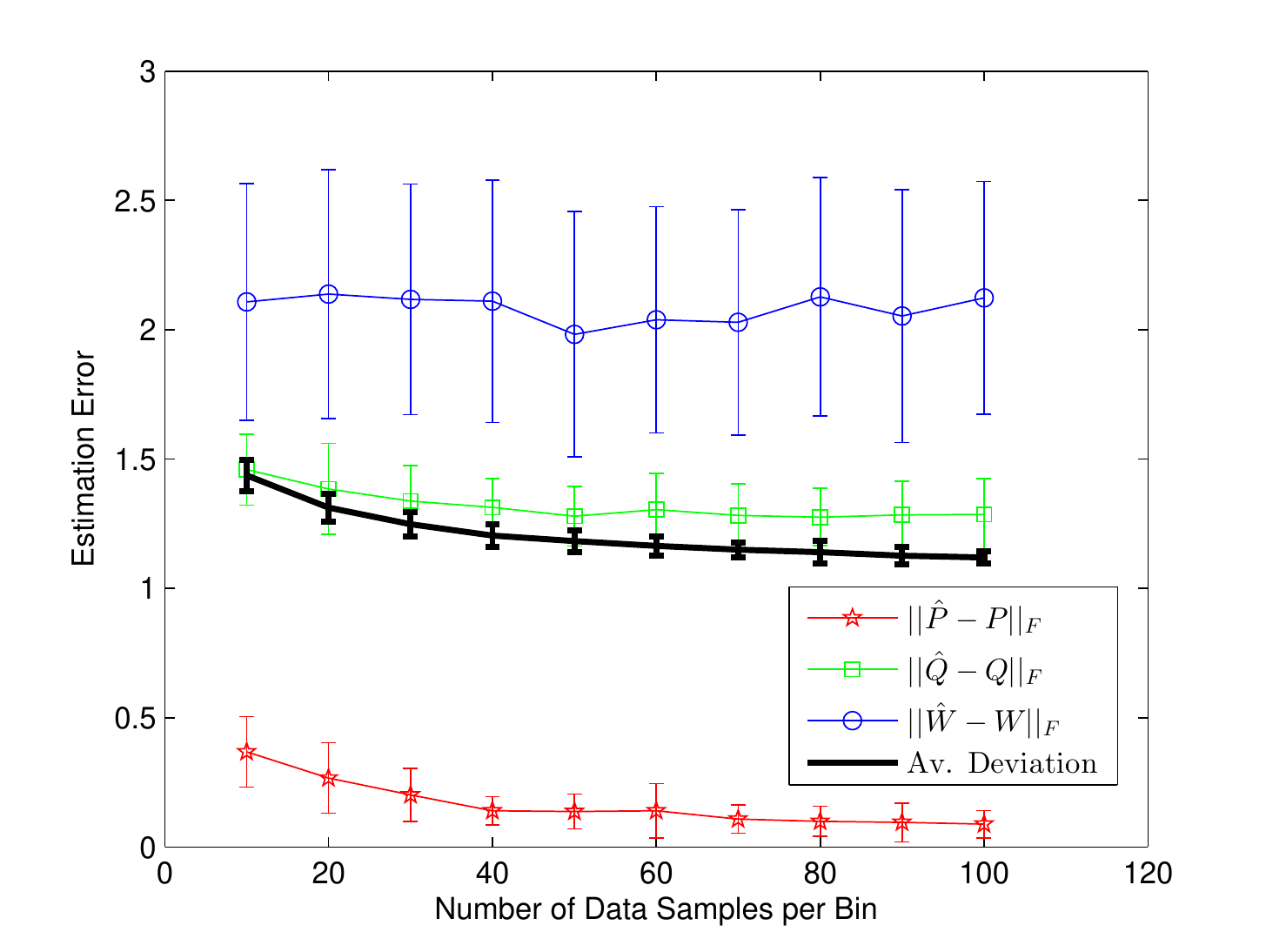}
    \caption{Average Deviation for increasing number of data samples, when $K = 5$, $L = 2$, $N = 5$ and 100 Monte-Carlo runs.}
    \label{Fig: Result-2}
\end{figure}

In Figure \ref{Fig: Result-2}, we run this experiment for different number of i.i.d. data samples per bin, when there are $N = 5$ bins, $K = 5$ choices and $L = 2$ attributes. We repeat this experiment over 100 Monte Carlo runs and plot the average deviation $f(\hat{W},\hat{Q})$, along with $||P - \hat{P}||_F$, $||Q - \hat{Q}||_F$ and $||W - \hat{W}||_F$, where $\hat{P} = \hat{W}\hat{Q}$. We plot the estimation error for increasing number of data samples (number of the agent's choice revelations). Note that the average deviation $f(W,Q)$ and the estimation errors of $Q$ and $W$ matrices decrease, but do not converge to zero. This phenomenon can be attributed to the fact that our matrix factorization framework in Problem \eqref{Problem: General} does not have a unique solution.

%


In the future, we will study special structures of $Q$ and $W$ matrices that guarantee the existence of a unique solution to Problem \eqref{Problem: General}, since the average deviation $f(W,Q)$ does not converge to zero in general. Furthermore, we will also investigate persuasive mechanisms where the controller constructs signals based on the agent's estimated preferences to steer their decisions in a desired manner. We will also study the impact of information-framing, where a given message is presented relative to some reference.


\bibliographystyle{IEEEtran}
\bibliography{references}

\begin{thebibliography}{10}
\providecommand{\url}[1]{#1}
\csname url@samestyle\endcsname
\providecommand{\newblock}{\relax}
\providecommand{\bibinfo}[2]{#2}
\providecommand{\BIBentrySTDinterwordspacing}{\spaceskip=0pt\relax}
\providecommand{\BIBentryALTinterwordstretchfactor}{4}
\providecommand{\BIBentryALTinterwordspacing}{\spaceskip=\fontdimen2\font plus
\BIBentryALTinterwordstretchfactor\fontdimen3\font minus
  \fontdimen4\font\relax}
\providecommand{\BIBforeignlanguage}[2]{{%
\expandafter\ifx\csname l@#1\endcsname\relax
\typeout{** WARNING: IEEEtran.bst: No hyphenation pattern has been}%
\typeout{** loaded for the language `#1'. Using the pattern for}%
\typeout{** the default language instead.}%
\else
\language=\csname l@#1\endcsname
\fi
#2}}
\providecommand{\BIBdecl}{\relax}
\BIBdecl

\bibitem{Samuelson1938}
P.~A. Samuelson, ``{A Note on the Pure Theory of Consumer's Behaviour},''
  \emph{Economica}, vol.~5, no.~17, pp. 61--71, 1938.

\bibitem{Book-VonNeumann}
J.~V. Neumann and O.~Morgenstern, \emph{Theory of games and economic
  behavior}.\hskip 1em plus 0.5em minus 0.4em\relax Princeton University Press,
  1944.

\bibitem{Book-Rubinstein1998}
A.~Rubinstein, \emph{{Modeling Bounded Rationality}}.\hskip 1em plus 0.5em
  minus 0.4em\relax MIT press, 1998.

\bibitem{Book-Train}
K.~Train, \emph{{Discrete Choice Methods with Simulation}}.\hskip 1em plus
  0.5em minus 0.4em\relax Cambridge university press, 2003.

\bibitem{Afriat1967}
S.~N. Afriat, ``{The Construction of Utility Functions from Expenditure
  Data},'' \emph{International Economic Review}, vol.~8, no.~1, pp. 67--77,
  1967.

\bibitem{McFadden1973}
D.~L. McFadden, ``{Conditional Logit Analysis of Qualitative Choice
  Behavior},'' in \emph{Frontiers of Econometrics}, P.~Zarembka, Ed.\hskip 1em
  plus 0.5em minus 0.4em\relax New York: Academic Press, 1973, pp. 105--142.

\bibitem{Rusmevichientong2014}
P.~Rusmevichientong, D.~Shmoys, C.~Tong, and H.~Topaloglu, ``{Assortment
  Optimization under the Multinomial Logit Model with Random Choice
  Parameters},'' \emph{Production and Operations Management}, vol.~23, no.~11,
  pp. 2023--2039, 2014.

\bibitem{Kallus2016}
N.~Kallus and M.~Udell, ``{Revealed Preference at Scale: Learning Personalized
  Preferences from Assortment Choices},'' in \emph{Proceedings of the 2016 ACM
  Conference on Economics and Computation}, ser. EC '16, 2016, pp. 821--837.

\bibitem{Tversky1969}
A.~Tversky, ``{Intransitivity of Preferences},'' \emph{Psychological Review},
  vol.~76, no.~1, pp. 31--48, January 1969.

\bibitem{Tversky1969b}
A.~Tversky and J.~E. Russo, ``{Substitutability and Similarity in Binary
  Choices},'' \emph{Journal of Mathematical Psychology}, vol.~6, no.~1, pp.
  1--12, 1969.

\bibitem{Hellwig2002}
C.~Hellwig, ``{Public Information, Private Information, and the Multiplicity of
  Equilibria in Coordination Games},'' \emph{Journal of Economic Theory}, vol.
  107, no.~2, pp. 191--222, 2002.

\bibitem{Block1960}
H.~Block and J.~Marschak, ``{Random Orderings and Stochastic Theories of
  Response},'' in \emph{{Contributions to Probability and Statistics}}, H.~M.
  Olkin, Ghurye and Mann, Eds.\hskip 1em plus 0.5em minus 0.4em\relax {Stanford
  University Press}, 1960, pp. 97--132.

\bibitem{Jagabathula2011}
S.~Jagabathula and D.~Shah, ``{Inferring Rankings Using Constrained Sensing},''
  \emph{IEEE Transactions on Information Theory}, vol.~57, no.~11, pp.
  7288--7306, Nov 2011.

\bibitem{Farias2012}
V.~F. Farias, S.~Jagabathula, and D.~Shah, ``{Sparse Choice Models},'' in
  \emph{{46th Annual Conference on Information Sciences and Systems (CISS)}},
  Princeton University, Princeton, NJ, USA., 2012, pp. 1--28.

\bibitem{Farias2013}
------, ``{A Nonparametric Approach to Modeling Choice with Limited Data},''
  \emph{Management Science}, vol.~59, no.~2, pp. 305--322, 2013.

\bibitem{Wang2013b}
Y.~X. Wang and Y.~J. Zhang, ``{Nonnegative Matrix Factorization: A
  Comprehensive Review},'' \emph{IEEE Transactions on Knowledge and Data
  Engineering}, vol.~25, no.~6, pp. 1336--1353, June 2013.

\bibitem{Blei2003}
D.~M. Blei, A.~Y. Ng, and M.~I. Jordan, ``{Latent Dirichlet Allocation},''
  \emph{Journal of Machine Learning Research}, vol.~3, no. Jan, pp. 993--1022,
  2003.

\bibitem{Arora2012}
S.~Arora, R.~Ge, and A.~Moitra, ``{Learning Topic Models -- Going Beyond
  SVD},'' in \emph{Proceedings of the 2012 IEEE 53rd Annual Symposium on
  Foundations of Computer Science}, ser. FOCS '12, Washington, DC, USA, 2012,
  pp. 1--10.

\bibitem{Adams2016}
\BIBentryALTinterwordspacing
C.~P. Adams, ``{Stochastic Matrix Factorization},'' \emph{SSRN Electronic
  Journal}, p. 1–24, 2016. [Online]. Available:
  \url{https://ssrn.com/abstract=2840852}
\BIBentrySTDinterwordspacing

\bibitem{Luo2017}
M.~Luo, F.~Nie, X.~Chang, Y.~Yang, A.~Hauptmann, and Q.~Zheng, ``{Probabilistic
  Non-Negative Matrix Factorization and Its Robust Extensions for Topic
  Modeling},'' in \emph{Proceedings of the 31st AAAI Conference on Artificial
  Intelligence}, 2017.

\end{thebibliography}

\end{document}